%% file: ijcai26.tex
\documentclass{article}
\usepackage{ijcai26}

\usepackage{times}
\usepackage{soul}
\usepackage{url}
\usepackage[hidelinks]{hyperref}
\usepackage[utf8]{inputenc}
\usepackage[small]{caption}
\usepackage{graphicx}
\usepackage{amsmath}
\usepackage{amsthm}
\usepackage{booktabs}
\usepackage{algorithm}
\usepackage[noend]{algorithmic}
\usepackage[switch]{lineno}
\usepackage{natbib}

\newtheorem{example}{Example}
\newtheorem{theorem}{Theorem}

\title{Probabilistic Verification of Recurrent Neural Networks for\\ Single and Multi-Agent Reinforcement Learning}

\author{
Luca Marzari$^1$
\And
Enrico Marchesini$^2$\\
\affiliations
$^1$TU Wien, Vienna, Austria\\
$^2$Massachusetts Institute of Technology, Cambridge (MA), USA\\
\emails
luca.marzari@tuwien.ac.at,
emarche@mit.edu
}

\input{packages}

\begin{document}

\maketitle

\begin{abstract}
    History-dependent policies induced by recurrent neural networks (RNNs) rely on latent hidden state dynamics, making verification in partially observable reinforcement learning (RL) challenging. Existing RNN verification tools typically rely on restrictive modeling assumptions or coarse over-approximations of the hidden state space, which can lead to overly conservative or inconclusive results. We propose \textbf{RNN} \textbf{Pro}babilistic \textbf{Ve}rification (\texttt{RNN-ProVe}), a probabilistic framework that \emph{estimates the likelihood} of undesired behaviors in RNN-based policies. \texttt{RNN-ProVe} uses policy-driven sampling to approximate the set of hidden states that are feasible under a trained policy, and derives statistical error bounds to produce bounded-error, high-confidence estimates of behavioral violations. Experiments on partially observable single-agent and cooperative multi-agent tasks show that \texttt{RNN-ProVe} yields more quantitative, feasibility-aware probabilistic guarantees than existing tools, while scaling to recurrent and multi-agent settings.

\end{abstract}

\input{sections/introduction}
\input{sections/preliminaries}

\input{sections/method}
\input{sections/experiments}
\input{sections/discussion}

\clearpage


\bibliographystyle{named}
\bibliography{ijcai26}

\clearpage

\onecolumn
\input{sections/appendix}

\end{document}

%% file: packages.tex
\usepackage{amsmath}
\usepackage{amssymb}
\usepackage{mathrsfs}
\usepackage{amscd}
\usepackage{amsfonts}

\usepackage{lipsum}  

\usepackage{enumitem}   
\usepackage{wrapfig}    
\usepackage{xcolor}
\usepackage{etoolbox}
\usepackage{tcolorbox}
\usepackage{makecell}
\usepackage{pifont}
\usepackage{graphics,color}
\usepackage{dsfont}
\usepackage[table]{xcolor}
\usepackage[bottom]{footmisc} 

\newtheorem{definition}{\textbf{Definition}}

\newtheorem{lemma}{\textbf{Lemma}}

\usepackage{tikz,ifthen}
\usetikzlibrary{arrows,trees,backgrounds,automata,shapes,decorations,plotmarks,fit,calc,positioning,shadows,chains}

\definecolor{color0}{RGB}{240, 78, 64} 
\definecolor{color4}{RGB}{60, 220, 125} 
\definecolor{color6}{RGB}{120, 100, 200} 
\definecolor{color7}{RGB}{107, 100, 200} 

\definecolor{nnedgecolor}{RGB}{90,90,90}
\tikzstyle{every pin edge}=[<-,shorten <=1pt]
\tikzstyle{every path}=[draw=color7!50]
\tikzstyle{neuron}=[circle,fill=black!25,minimum size=17pt,inner sep=0pt]
\tikzstyle{input neuron}=[neuron, fill=color4]
\tikzstyle{output neuron}=[neuron, fill=color0]
\tikzstyle{hidden neuron}=[neuron, fill=color6]
\tikzstyle{annot} = [text width=4em, text centered]
\tikzstyle{nnedge} = [-{stealth},shorten >=0.1cm, shorten <=0.05cm,line 
width=0.8pt,nnedgecolor]
\tikzstyle{nnedge_t} = [-{dashed},shorten >=0.1cm, shorten <=0.05cm,line 
width=0.8pt,nnedgecolor]
\usetikzlibrary{calc}

%% file: sections/introduction.tex
\section{Introduction}
\label{sec:introduction}
Recurrent neural networks (RNNs) are widely used in reinforcement learning (RL) to handle partial observability and multiple agents \citep{aydeniz_2024, valuefact}. Such RNN-based policies learn hidden-state dynamics that encode past observations and actions, making decision-making depend on the interaction \emph{history} rather than only the current input. Verifying such policies is challenging, as most existing verification tools are designed for feedforward neural networks \citep{LiuSurvey, ModelVerification}.

When applied to RNNs, verification is typically restricted to robustness analysis, where networks are checked over a single transition by bounding the effect of perturbations to a fixed input on the corresponding output. While effective for local robustness, this abstraction fails to capture the sequential and history-dependent nature of recurrent RL agents. In particular, assessing whether an RNN-based policy performs unsafe or undesired actions (or behaviors, interchangeably) requires reasoning about which hidden states can arise from \emph{feasible, policy-induced histories}, that is, interactions the agent can actually encounter during execution (Fig.~\ref{fig:overview}). Existing approaches addressing this challenge typically rely on explicit transition models to reason about recurrent policies, and verification remains computationally challenging even under these assumptions \citep{akintunde2019verification,carr2021verifiable,everett2021reachability}. Moreover, these methods are limited in both scalability and expressivity: they fail to scale to multiple agents or high-dimensional spaces \citep{marchesini2025rl2grid, marl2grid_iclr2026}, and often yield only binary certificates (e.g., \emph{safe or unsafe}). Such binary outcomes are frequently uninformative in practice, as policies may differ substantially in how often violations occur across different feasible histories. In principle, sound verification of RNN-based policies would require enumerating all reachable hidden-state configurations during execution and checking whether any of them lead to undesired behavior. However, this requirement is equivalent to computing neural network preimage measures for which a desired behavior holds \citep{INVPROP, zhang2024provable}, a problem that is \#P-hard due to the nonlinearity and nonconvexity of neural networks \citep{CountingProve}. Hence, under standard complexity assumptions, obtaining an exact characterization of the feasible-history space is computationally infeasible. This observation suggests that, for recurrent policies in RL, exact worst-case verification is not only computationally prohibitive but also misaligned with the structure of history-dependent decision making.

\begin{figure}[t]
    \centering
    \includegraphics[width=.9\linewidth]{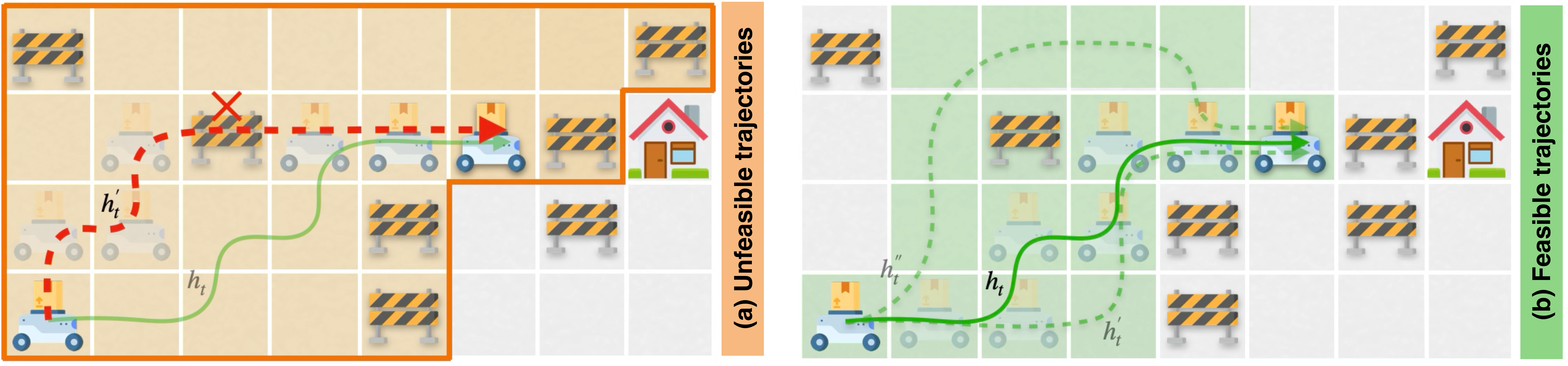}
    \caption{(a) Interval-based verification over-approximates the hidden state space, admitting infeasible histories. (b) \texttt{RNN-ProVe} assesses policies over probabilistically feasible histories, yielding more precise and informative certificates.}
    \label{fig:overview}
    \vspace{-5pt}
\end{figure}

These limitations motivate a probabilistic and quantitative perspective on verification for RNN-based policies. Rather than enumerating all feasible hidden states exactly, we aim to estimate the likelihood of undesired behavior by quantifying how much of the feasible-history space leads to behavioral violations under a trained policy. To this end, we introduce \textbf{RNN} \textbf{Pro}babilistic \textbf{Ve}rification (\texttt{RNN-ProVe}), a probabilistic framework that approximates the policy-induced feasible hidden state space using policy-driven sampling with explicit statistical error bounds, and produces bounded-error, high-confidence verification certificates. At its core, \texttt{RNN-ProVe} learns a \emph{feasibility oracle} that distinguishes feasible from infeasible hidden state representations with respect to the distribution of histories induced by the policy during training. The oracle is implemented as a classifier trained on feasible representations encoded by the policy. \texttt{RNN-ProVe} then performs quantitative verification over probabilistically feasible histories identified by the classifier. Additionally, we derive statistical confidence bounds to determine the required sample sizes for verification, ensuring that the resulting certificates have provable error bounds and high statistical confidence.

We evaluate \texttt{RNN-ProVe} on single-agent and cooperative multi-agent partially observable tasks, demonstrating that the framework scales to settings in which existing verification tools become overly conservative or inapplicable, while providing informative probabilistic assessments of behavioral risk. Our results suggest that probabilistic verification over feasible histories is not merely a heuristic relaxation, but a principled response to fundamental computational barriers. These outcomes motivate the following contributions:
\begin{itemize}[noitemsep]
    \item We formulate the problem of \emph{probabilistic verification} of recurrent RL policies under partial observability, explicitly accounting for the set of hidden states that are feasible under a trained policy.
    \item We propose \texttt{RNN-ProVe}, a policy-driven probabilistic verification framework that combines feasibility learning with Monte Carlo estimation to quantitatively assess behavioral violations in RNN-based policies.
    \item We extend the problem to fully cooperative multi-agent RL (MARL) and derive bounded-error, high-confidence guarantees for feasible-history identification and violation estimation, yielding a tractable approximation to an otherwise \#P-hard verification problem.
\end{itemize}

More broadly, this work advocates probabilistic verification over feasible histories as a principled foundation for reasoning about safety and reliability in history-dependent decision-making systems.


%% file: sections/preliminaries.tex
\section{Preliminaries and Related Work}\label{sec:preliminaries}

At time $t$, an RNN produces an output $y_t$ by using a hidden state representation $h_t$ of an input sequence $\mathbf{x} = (x_1, \dots, x_t)$.
\begin{figure}[b]
    \centering
    \vspace{-5pt}
    \includegraphics[width=0.7\linewidth]{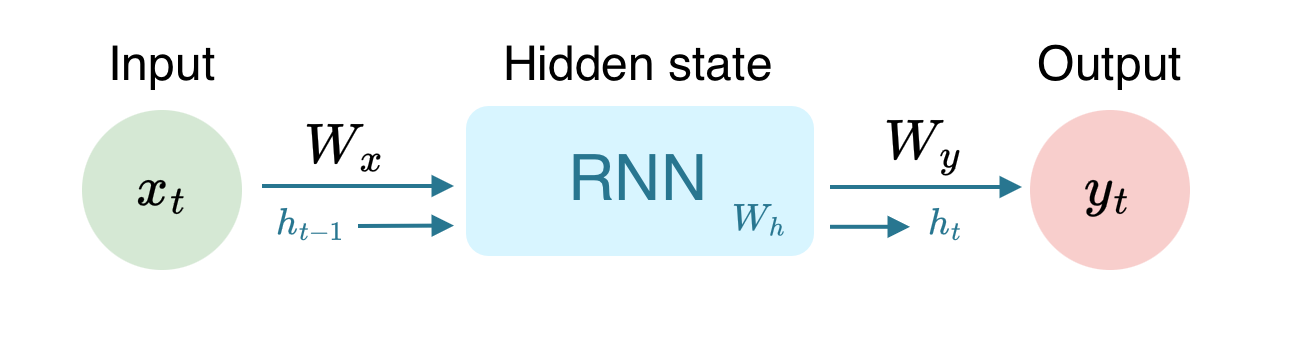}
    \vspace{-8pt}
    \caption{An explanatory RNN transition.}
    \label{fig:rnn}
\end{figure}
Following the illustration of Fig. \ref{fig:rnn}, the transition function for a vanilla RNN is $h_t = \sigma(W_{x} x_t + W_{h} h_{t-1} + b_h)$, where $W_{x},  W_{h}$ are the input and recurrent weight matrices, $b_h$ is a bias vector, and $\sigma$ is typically a non-linear activation function. The output is derived via a linear transformation as $y_t = W_{y} h_t + b_y$, where $W_{y}$ is the output weight matrix and $b_y$ is another bias vector. This recursive structure enables RNNs, as well as gated variants such as LSTMs \citep{LSTM} and GRUs \citep{GRU}, to model temporal dependencies over input sequences.

\paragraph{RNNs in RL.}
In partially observable RL, an agent typically learns a policy $\pi : \mathcal{H} \rightarrow \Delta(\mathcal{A})$; where $\mathcal{H}$ denotes the agent's action-observation history $(o_0, a_0, \ldots, o_{t-1})$, $\mathcal{A}$ is the set of actions modeling the RNN's output $y$, and $o \in \mathcal{O}$ are observations (i.e., RNN's input $x$) distributed following a stochastic transition function $T_{\mathcal{O}} : \mathcal{S} \times \mathcal{A} \to \Delta(\mathcal{O})$ (where $\mathcal{O}, \mathcal{S}$ are the finite sets of observations and states, respectively) \citep{pomdp}. In practice, at time $t$, $\pi$ uses the RNN's hidden state $h_t$ as a compressed history representation that depends on: (i) the previous hidden state $h_{t-1}$; and (ii) the current observation $o_t$, to output an action.\footnote{For simplicity, we use $h, \mathcal{H}$ for both histories and hidden states.}
When tackling problems with $\mathcal{N}$ agents, we consider the setting where each agent $i \in \mathcal{N}$ has a local observation space $\mathcal{O}^i$ and conditions its policy on the local action-observation history \citep{decpomdp}. To achieve this decentralization, centralized training with decentralized execution (CTDE) is the de facto standard paradigm \citep{amato_survey, aydeniz_2025}. In CTDE, agents leverage privileged information during training (including the state $s \in \mathcal{S}$) while learning local history-based policies for execution~\citep{papoudakis2021benchmarking}. Similarly, \texttt{RNN-ProVe} accesses the state while keeping decentralized execution. Thus, our framework is consistent with established practice in partially observable RL.

\paragraph{RNN Verification.} Although RNNs are commonly used in partially observable RL tasks, existing RNN verification methods are developed primarily for robustness analysis. Robustness verification certifies the stability of network outputs under bounded input perturbations over a fixed time horizon.

Without loss of generality, consider verifying an RNN with a single scalar output, where the robustness requirement is to ensure that the output satisfies a given input-output property \citep{LiuSurvey}. Let $f$ denote the RNN transition dynamics such that $h_t = f(h_{t-1}, x_t)$ for $t \in \{1,\ldots,T\}$, where $T$ is the verification horizon. The property is encoded by appending a linear layer that outputs a margin value \citep{bcrown}. Specifically, the final output is defined as $y_T = c^\top h_T + b_y$, where $y_T > 0$ indicates that the property is satisfied. In more detail, let $\mathcal{H}_0 = \{h_0\}$ be the initial hidden state set and let $\mathcal{X}_{1:T} = \{\mathcal{X}_1, \ldots, \mathcal{X}_T\}$ denote the admissible input sets, where each $\mathcal{X}_t = [\underline{x}_t, \overline{x}_t]$ represents an $\ell_\infty$-ball of radius $\varepsilon$ around a nominal input. The over-approximated reachable hidden state sets are computed recursively as $\mathcal{H}_t = \{ f(h, x) \mid h \in \mathcal{H}_{t-1},\ x \in \mathcal{X}_t \},$
and the corresponding output set is $\mathcal{Y}_T = \{ y_T \mid h \in \mathcal{H}_T \}$.\footnote{We note that these reachability-based methods do not consider which hidden states can arise from feasible, policy-induced histories and can thus produce misleading outcomes.} Robustness verification then amounts to checking whether $\min(\mathcal{Y}_T) > 0$.\footnote{We provide a numerical example in Appendix \ref{appendix:numerical_example}.}

\begin{definition}[Robustness \textsc{RNN-Verification}]\label{def:rnn_ver_robustness}

\phantom{a}

    {\bf Input}: A tuple 
    $\mathcal{T}=\langle f, \mathcal{H}_0, \mathcal{X}_{1:T}, \pi\rangle$.
    
    {\bf Output}:   $\texttt{robust} \iff 
    \forall y \in \mathcal{Y}_T, y > 0$.\\ Equivalently, this requires proving that $\min(\mathcal{Y}_T) >0$.
\end{definition}

\paragraph{Related Work.}
Several methods have been proposed for RNNs robustness verification. \cite{akintunde2019verification} unrolls the network over a finite time horizon, reducing it to a feedforward architecture. While effective, this strategy incurs exponential complexity in the horizon length. \cite{jacoby2020verifying} infers inductive invariants over hidden states to remove explicit horizon dependence. However, synthesizing precise invariants for nonlinear networks is computationally difficult and often yields conservative results. To improve scalability, interval- and relaxation-based reachability analyses are used within branch-and-bound frameworks \citep{ko2019popqorn, babSplitRelu, acrown, mohammadinejad2021diffrnn, du2021cert, tran2023rnnreach, ECAI25}. These methods over-approximate reachable hidden states by assuming step-wise bounded input perturbations and targeting local robustness properties. Unlike robustness-based approaches that estimate output stability under bounded or stochastic input perturbations, our work quantifies how often undesired behaviors arise over the set of \emph{policy-induced feasible histories}, addressing a fundamentally different source of uncertainty. This gap increases over long horizons, as over-approximation errors accumulate.

Beyond robustness, some works aim to verify recurrent policies. \citet{carr2021verifiable} propose extracting finite-state automata or abstract MDPs from trained RNN policies. While conceptually powerful, these approaches rely on discretization heuristics that scale exponentially with the dimensionality of the hidden state, restricting their applicability to small-scale RNNs. \citet{everett2021reachability} perform backward reachability analysis directly in the state-action space. This method assumes full knowledge of the environment dynamics and deterministic transitions, assumptions that are rarely satisfied in complex, partially observable RL settings. Moreover, maintaining exact or even over-approximated reachable sets becomes intractable as the trajectory length increases.

Hence, existing approaches rely on strong modeling assumptions or suffer from severe scalability issues due to the combinatorial explosion inherent in abstracting continuous hidden state spaces. These problems highlight that current methods cannot effectively verify recurrent policies in RL.

%% file: sections/method.tex
\section{RNN-Verification for RL}
To verify RNN-based RL policies, we must move beyond input-perturbation-based \textsc{RNN-Verification} (Def. \ref{def:rnn_ver_robustness}) and reason about uncertainty in the RNN’s internal memory, namely the hidden state that encodes the agent’s history. Specifically, instead of checking robustness against noisy inputs, we must verify actions over the set of feasible hidden states that could have arisen from past interactions. Hence, given an RNN-based policy $f$ and the current observation $o$, our goal is to determine whether there exists any hidden state $h$, within the set of feasible ones $\mathcal{H}^* \subseteq \mathcal{H}$ up to $o$, such that $f(h,o)$ leads to a violation of a desired behavior (e.g, an action that leads to an unsafe situation, such as colliding into an obstacle). We formally define this verification problem as:

 \begin{definition}[\textsc{RNN-Verification} for RL]\label{def:drl_rnn_safety}
 \phantom{a}

{\bf Input}: A tuple $\mathcal{T}=\langle f, \mathcal{H}, o \rangle$.

{\bf Output}: $\texttt{violation} \iff \exists h \in \mathcal{H}^* \;\vert\; f(h, o) \leq 0$.
 \end{definition}


As in robustness verification, we assume that the behavior of interest can be expressed using a single scalar output whose sign indicates whether a violation occurs (i.e., the output is strictly positive if and only if the undesired action is not selected). This assumption can be enforced by appending a post-processing layer that encodes the action selection mechanism together with the desired behavioral constraint.
In summary, Fig. \ref{fig:RL_safety_prop} illustrates our verification problem, which checks whether there exists a feasible hidden state $h \in \mathcal{H}^*$ such that, when combined with the current observation $o$, the RNN selects an action that violates a desired behavior. 

\begin{figure}[t]
    \centering
    \includegraphics[width=.9\linewidth]{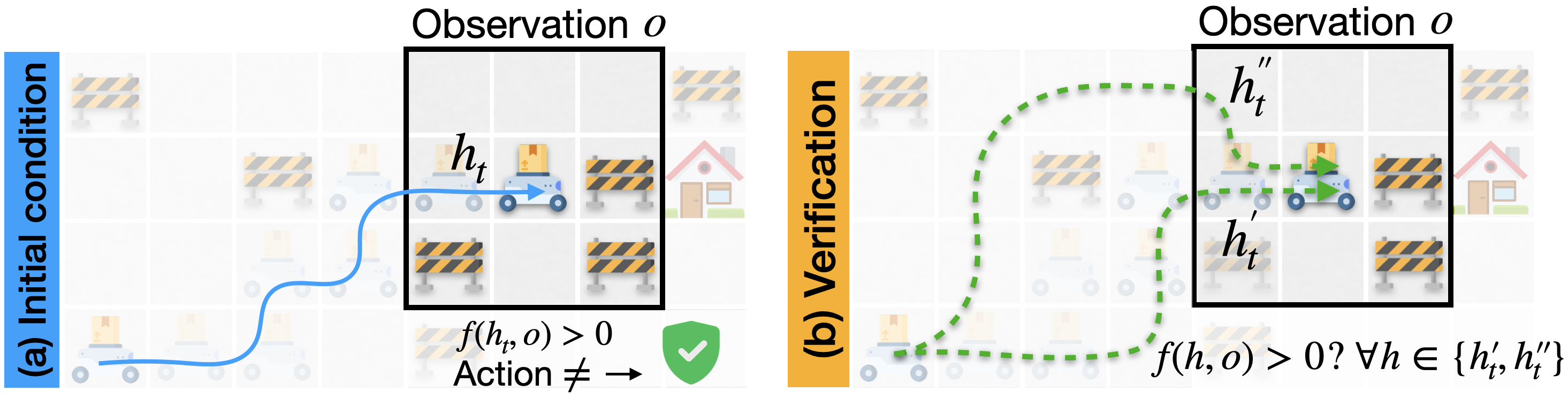}
    \caption{\textsc{RNN-Verification} for RL navigation: (a) The agent selects a safe action that avoids a collision. (b) The verification problem asks whether there exists a feasible history $h$ (e.g., $h_t', h_t''$) such that the agent selects an undesired action, i.e., $f(h,o)\leq 0$.}
    \vspace{-5pt}\label{fig:RL_safety_prop}
\end{figure}

In contrast to robustness, \textsc{RNN-Verification} for RL requires characterizing the set of feasible hidden states $\mathcal{H}^*$ that reach a certain configuration (state). To this end, formal verification methods typically use coarse geometric over-approximations, such as independent intervals or hyperrectangles. However, as previously stated, these relaxations may admit spurious hidden states (i.e., combinations of hidden neuron values that cannot arise from any valid interaction history of an RL agent). As a result, propagating such inflated sets through the transition function $f$ can yield spurious counterexamples, or false positives, that do not correspond to situations the agent can actually encounter. Sec. \ref{sec:rnn_prove} discusses how \texttt{RNN-ProVe} addresses this issue.

\paragraph{Quantifying violations.} Existing RNN verification tools typically return only a binary outcome and do not quantify the extent to which a trained RNN violates a desired requirement. In the context of RNN-based RL policies, this means that it is not possible to measure how much of the feasible hidden state space leads to violations of a desired behavior.

This limitation motivates our novel quantitative formulation, \textsc{\#RNN-Verification}. In analogy with the feedforward network case \citep{INVPROP, aamas23, TIST}, we invert the analysis perspective. Rather than characterizing all the inputs of the RNN that lead to violating the desired behavior, we characterize the set of feasible hidden states (histories) that violate the behavior for a certain situation at time $t$ (i.e., for a fixed input observation). This formulation enables explicit reasoning about uncertainty over feasible histories and yields a characterization of the region of hidden states that induce undesired behavior.

 \begin{definition}[\textsc{\#RNN-Verification} for RL]\label{def:sharp_rnn}
 
\phantom{a}

    {\bf Input}: A tuple 
    $\mathcal{T}=\langle f, \mathcal{H}, o \rangle$.
    
    {\bf Output}:  $Vol(\Gamma(\mathcal{T}))$.
    \vspace{1.5mm}
    
\noindent With $\Gamma(\mathcal{T}):= \{h \in \mathcal{H}^* \;\vert\; f(h, o) \leq 0\} \subseteq \mathcal{H}$ the set of feasible histories for a fixed $o$ that result in undesired behaviors.
 \end{definition}
 
 Broadly speaking, the objective of this quantitative formulation is to compute the volume of $\Gamma(\mathcal{T})$ relative to $\mathrm{Vol}(\mathcal{H})$, which corresponds to the probability that the current policy exhibits an undesired behavior. 
We also note that exactly characterizing the feasible-history space $\mathcal{H}^*$ is at least as hard as computing the volume of inputs that violates a behavior for a neural network, a problem known as \textsc{\#DNN-Verification} \citep{CountingProve}.

\begin{theorem}\label{thm:complexity}
    The \textsc{\#RNN-Verification} problem is \#P-hard.
\end{theorem}

In Appendix \ref{appendix:hardness} we report a complete argument of this theoretical result. Under standard complexity assumptions, it implies that for any non-trivial RNN-based RL policy, obtaining an exact characterization of $\mathrm{Vol}(\Gamma(\mathcal{T}))$ is computationally infeasible. This observation motivates our shift to a probabilistic perspective on verification, which we develop in Sec. \ref{sec:rnn_prove}.

\paragraph{Verifying Cooperative Multi-Agent RL.}
To the best of our knowledge, we now extend the previous problem formulations for the first time to fully cooperative MARL. We consider a setting in which a decentralized set of $\mathcal{N}$ agents must coordinate to achieve a shared objective.
In this setting, each agent $i \in \mathcal{N}$ acts according to its own RNN-based policy $f_i$, relying on its local observation $o_i$ and hidden state representation $h_i\in \mathcal{H}_i$. Considering the agents' common goal, we define verification of a cooperative MARL problem as the conjunction of individual desired-behavior constraints, meaning the system does not exhibit violations if every agent satisfies its individual desired behavior. Conversely, the system has violations if \textit{at least one} agent violates its behavior. We formally define the verification problem for this setting as follows:

\begin{definition}[\textsc{RNN-Verification} for MARL]\label{def:marl_rnn_safety}
\phantom{a}

{\bf Input}: $\mathbf{T} = \{ \mathcal{T}_1, \ldots, \mathcal{T}_\mathcal{N} \}$, where $\mathcal{T}_i = \langle f_i, \mathcal{H}_i, o_i \rangle$.

{\bf Output}: $\texttt{violation} \iff \exists i \in \{1, \ldots, \mathcal{N}\}, \newline \text{\hspace{4.6cm}} \exists h_i \in \mathcal{H}_i^* \;\vert\; f_i(h_i, o_i) \leq 0$.
\end{definition}

We thus extend the quantitative formulation of Def. \ref{def:sharp_rnn} to the multi-agent domain. Due to our interest in fully cooperative MARL, we are interested in quantifying the agent with the highest probability of violations given its current local history. Therefore, the \textsc{\#RNN-Verification} problem for MARL aims to compute the maximum volume of undesired feasible histories across all agents.

\begin{definition}[\textsc{\#RNN-Verification} problem for MARL]\label{def:sharp_rnn_marl}
\phantom{a}

    {\bf Input}: $\mathbf{T} = \{ \mathcal{T}_1, \ldots, \mathcal{T}_\mathcal{N} \}$, where $\mathcal{T}_i = \langle f_i, \mathcal{H}_i, o_i \rangle$.
    
    {\bf Output}: $\max_{i \in  \mathcal{N}} Vol(\Gamma(\mathcal{T}_i))$.
    \vspace{1.5mm}
    
\end{definition}

Since we adopt the standard decentralized setting in which each agent $i$ acts based solely on its local information, the evaluation of $Vol(\Gamma(\mathcal{T}_i))$ for each agent is independent. The verification certificate is therefore determined by the worst-performing agent. This cooperative MARL formulation preserves tractability by decomposing into independent verification tasks, enabling parallel analysis without introducing additional conservatism.
Accordingly, in the remainder of the paper, we present the proposed probabilistic framework and its theoretical guarantees for the single-agent case.

\subsection{RNN-ProVe: a Novel Probabilistic Approach}\label{sec:rnn_prove}

Following the result in Theorem \ref{thm:complexity}, \texttt{RNN-ProVe} solves a relaxed probabilistic version of the \textsc{\#RNN-Verification} problem for RL. Intuitively, the following formalization asks for a scalar volume estimate that accurately approximates the volume measure of the true undesired set $\Gamma(\mathcal{T})$. 

\begin{definition}[Approximate \textsc{\#RNN-Verification} for RL]
\label{def:approx_rnn_ver}
\phantom{a}
    {\bf Input}: A tuple $\mathcal{T}=\langle f, \mathcal{H}, o \rangle$, an error tolerance $\epsilon \in (0,1)$, and a confidence parameter $\delta \in (0,1)$.
    
    {\bf Output}: $\tilde{V} \in [0, Vol(\mathcal{H})]$ estimate of $Vol(\Gamma(\mathcal{T}))$, s.t.
    \[
    \Pr\left(\left| \frac{\tilde{V} - Vol(\Gamma(\mathcal{T}))}{Vol(\mathcal{H})} \right| \leq \epsilon \right) \geq 1-\delta.
    \]
\end{definition}

This probabilistic problem definition requires that, with high confidence $1-\delta$, the deviation between the estimated $\tilde{V}$ and true volumes $Vol(\Gamma(\mathcal{T}))$, normalized by the total volume of the hidden space $Vol(\mathcal{H})$, remains within the specified error bound $\epsilon$. Since computing $Vol(\Gamma(\mathcal{T}))$ exactly is computationally prohibitive, an alternative strategy is to perform a Monte Carlo estimation as
$Vol(\Gamma(\mathcal{T})) \approx Vol(\mathcal{H})\cdot \frac{1}{k} \sum_{i=1}^k \mathds{1}_{f(h_i, o) \leq 0},$ where $h_1, \ldots, h_k$ are independent samples drawn from the hidden state domain $\mathcal{H}$, and $\mathds{1}_{f(h_i,o) \leq 0}$ is the indicator function for undesired behaviors. However, this naive sampling strategy implicitly assumes that all histories within the domain $\mathcal{H}$ are feasible, while the set of feasible histories $\mathcal{H}^*$ is a complex, lower-dimensional manifold within $\mathcal{H}$. Consequently, a uniform sample from $\mathcal{H}$ inevitably includes unfeasible histories, leading to inaccurate verification results (as we will show in our experiments).

To address this issue, one would require an oracle that determines whether, for a given configuration (i.e., a state $s \in \mathcal{S}$), a history $h$ is feasible, meaning that it can be realized as an internal hidden state by $\pi$. Since no exact analytical form of this oracle exists, \texttt{RNN-ProVe} leverages the universal approximation theorem \citep{hornik1989multilayer} to learn an approximation of this oracle, as illustrated in Fig. \ref{fig:rnn_prove_overview}. 
\begin{figure}[b]
    \centering
    \includegraphics[width=0.85\linewidth]{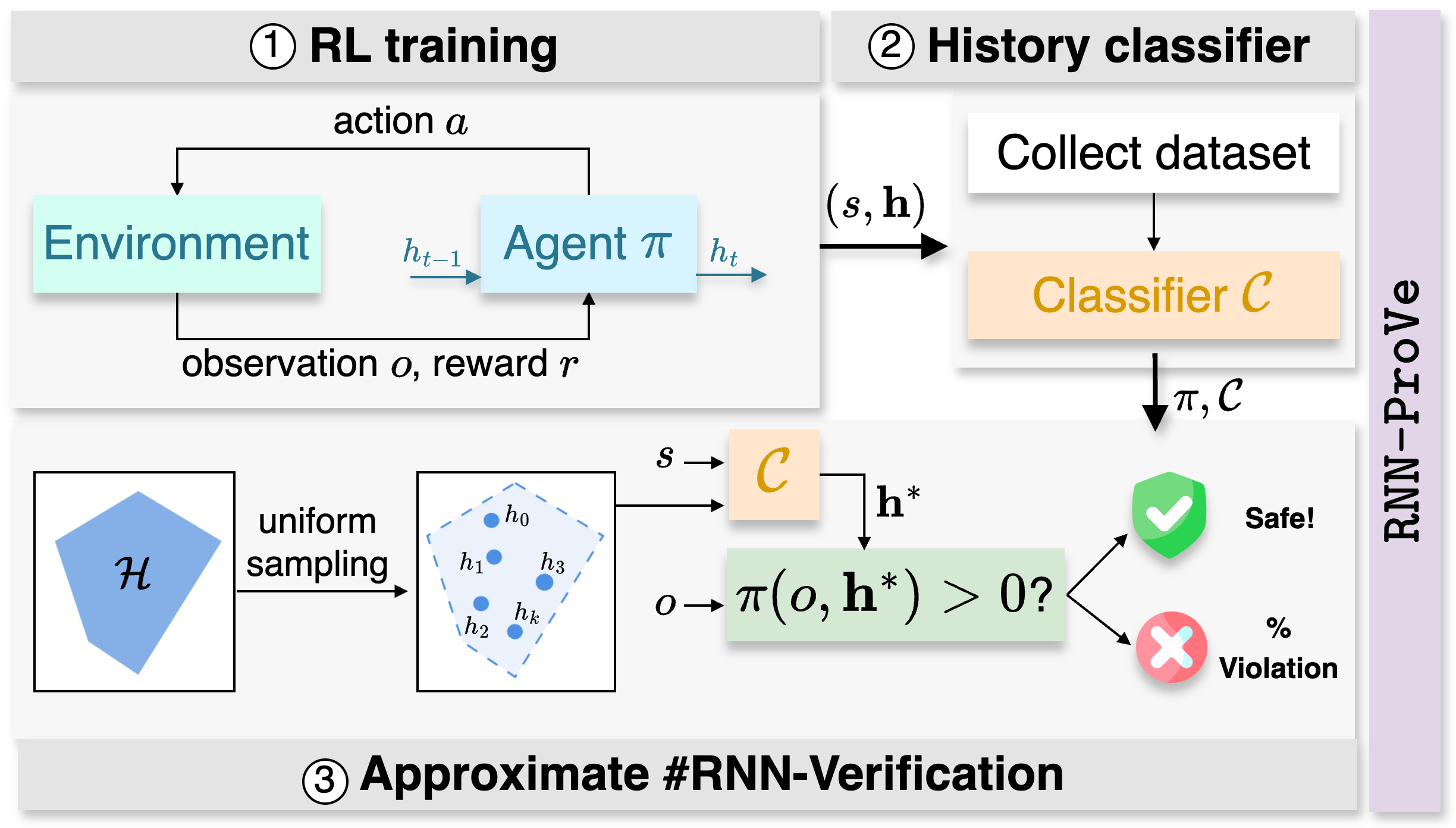}
    \caption{\texttt{RNN-ProVe} overview.}
    \label{fig:rnn_prove_overview}
    \vspace{-3mm}
\end{figure}
In more detail, during RL training, the agent's policy explores a wide range of environment states $s$ and generates corresponding internal hidden state representations of the actual histories leading to $s$. \texttt{RNN-ProVe} records these state-hidden state pairs to construct a dataset used to train a \emph{feasibility history classifier} $\mathcal{C}: \mathcal{S} \times \mathcal{H} \rightarrow \{0,1\}$. Specifically, given a state $s$ and a hidden state $h$, the classifier predicts whether $h$ corresponds to a feasible hidden state for the trained policy (phases 1 and 2 in the upper part of Fig. \ref{fig:rnn_prove_overview}). This data collection process leverages maximum entropy-based RL exploration mechanisms, which, in our experiments, lead to a sufficient coverage of the hidden state space for our dataset \citep{RL}. As a result, the dataset captures a wide variety of the feasible hidden state space, including both optimal and suboptimal but realizable interactions, which is essential for training a reliable feasibility classifier.
Once the classifier $\mathcal{C}$ is trained, \texttt{RNN-ProVe} enters the quantitative verification phase (phase 3 in Fig. \ref{fig:rnn_prove_overview}). For a given state of interest $s$ where we want to verify a certain behavior, the method samples $n$ candidate history vectors from the hidden state domain $\mathcal{H}$ and evaluates them using $\mathcal{C}$. The classifier acts as a feasibility oracle, retaining only those representations $\mathbf{h}^*$ that are predicted to be feasible hidden states of the RNN-based policy leading to state $s$. These validated histories are then used to evaluate the policy under the observation $o$ associated with $s$. Hence, for each $h \in \mathbf{h}^*$, the method checks whether the action selected by the trained policy satisfies the desired behavior (i.e., selects a certain action), that is, whether $f(h,o) > 0$. If violations are observed, \texttt{RNN-ProVe} estimates the undesired volume $\tilde{V}$ as the ratio of violating hidden states to the total number of validated hidden states tested.

The reliability of the resulting certificate depends on two sources of uncertainty: (i) the accuracy of the feasibility classifier, which determines whether the identified set faithfully represents the true feasible hidden state manifold without admitting spurious histories; and (ii) sampling variance, which depends on the number of validated hidden states used to estimate the undesired volume. \texttt{RNN-ProVe} addresses both sources of uncertainty in the next section through rigorous statistical analysis, providing high-confidence bounds and bounded-error certificates for feasible hidden state identification and the resulting verification results.

\subsection{Theoretical Guarantees}
\label{sec:theory}
\texttt{RNN-ProVe} guarantees rely on decomposing the estimation error $\epsilon$ into two components: (i) the approximation error of the histories classifier; and (ii) the statistical variance of the safety evaluation (i.e., the sampling error). We bound both terms using Hoeffding's inequality \citep{hoeffding1963probability} to derive a global probabilistic certificate. 

\begin{lemma}[Hoeffding's inequality {\citep{hoeffding1963probability}}]\label{lemma:hoeffding}
Let $X_1,\dots,X_n$ be independent random variables such that $X_i \in [0,1]$, and let $\mu = \mathbb{E}[X_i]$ and its estimate $\hat{\mu} = \frac{1}{n}\sum_{i=1}^n X_i.$
Then, for any $\epsilon > 0$, $\Pr\!\left( |\hat{\mu} - \mu| \ge \epsilon \right)
\le
2\exp\!\left( -2n\epsilon^2 \right)$.
\end{lemma}

In our setting, let, without loss of generality, $\mathcal{H}\subseteq [0,1]^n$ be the hidden state domain and $\epsilon, \delta \in (0,1)$ be the total error tolerance and confidence parameters required by Def. \ref{def:approx_rnn_ver}. We distribute these budgets into classifier validation components $(\hat{e}, \epsilon_{clf}, \delta_{clf})$ and verification components $(\epsilon_{ver}, \delta_{ver})$, such that $\epsilon = \hat{e} + \epsilon_{clf} + \epsilon_{ver}$ and $\delta = \delta_{clf} + \delta_{ver}$.

\begin{theorem}[\texttt{RNN-ProVe} Probabilistic Guarantees]
  \label{thm:pac_guarantee}
  Let $\hat{p}$ be the empirical probability of undesired behavior computed over $N$ histories validated via the classifier $\mathcal{C}$. Let $\hat{e}$ be the empirical error rate of $\mathcal{C}$ measured on a held-out validation set of size $M$. If the sample sizes satisfy $N \geq \frac{\ln(2/\delta_{ver})}{2\epsilon_{ver}^2}$ and $M \geq \frac{\ln(2/\delta_{clf})}{2\epsilon_{clf}^2}$,
  then the volume estimate $\tilde{V} = Vol(\mathcal{H}) \cdot \hat{p}$ computed by \texttt{RNN-ProVe} satisfies $\Pr\left(\left| \frac{\tilde{V} - Vol(\Gamma(\mathcal{T}))}{Vol(\mathcal{H})} \right| \leq \epsilon \right) \geq 1-\delta$.

\end{theorem}

\begin{proof}
    Let $p^*$ denote the true probability of undesired behavior (i.e., the $Vol(\Gamma(\mathcal{T}))/Vol(\mathcal{H})$) and $p_{\mathcal{C}}$ denote the probability induced by the classifier $\mathcal{C}$. We bound the total deviation $|\hat{p} - p^*|$ by applying the triangle inequality $|\hat{p} - p^*| \le |\hat{p} - p_{\mathcal{C}}| + |p_{\mathcal{C}} - p^*|$, representing the verification sampling and classifier error. To bound the classification error, we note that by applying Lemma \ref{lemma:hoeffding}, and enforcing that $2e^{-2M\epsilon_{clf}^2} \leq \delta_{clf}$, we obtain a validation set of size $M \ge \frac{\ln(2/\delta_{clf})}{2\epsilon_{clf}^2}$. By using this $M$ we guarantee, with probability at least $1-\delta_{clf}$, that $|p_{\mathcal C} - p^*| \leq \epsilon_{clf} \leq \epsilon_{clf}+\hat{e}$. Analogously, $\hat p$ is a Monte Carlo estimate of $p_{\mathcal C}$ obtained from $N$ independent samples. By applying once again Lemma \ref{lemma:hoeffding}, if $N \geq \frac{\ln(2/\delta_{ver})}{2\epsilon_{ver}^2}$ we have that $|\hat{p} - p_{\mathcal C}| \le \epsilon_{ver}$, with probability at least $1-\delta_{ver}$. By union bound, these inequalities hold simultaneously with probability at least $1-(\delta_{clf}+\delta_{ver}) = 1-\delta$, giving $|\hat{p}-p^*| \leq \epsilon_{ver}+\hat{e} + \epsilon_{clf} = \epsilon$. 
    Finally, substituting $\hat{p} = \frac{\tilde{V}}{\text{Vol}(\mathcal{H})}$ and $p^* = \frac{Vol(\Gamma(\mathcal{T}))}{Vol(\mathcal{H})}$, we obtain $\Pr \left( \left| \frac{\tilde{V}}{\text{Vol}(\mathcal{H})} - \frac{\text{Vol}(\Gamma(\mathcal{T}))}{\text{Vol}(\mathcal{H})} \right| \leq \epsilon \right) \geq 1-\delta$, concluding the argument.
\end{proof}

This result shows that \texttt{RNN-ProVe} provides a bounded-error, high-confidence certificate for the volume of hidden states violating the desired behavior, providing a solution to the approximate \textsc{\#RNN-Verification} for RL problem. 

For example, consider a target confidence of $1-\delta = 99.9\%$ and an $\epsilon = 5\%$ error in the certified violation volume.
Suppose the classifier $\mathcal{C}$ achieves an empirical
error rate of $\hat e = 0.01$ (i.e., $99\%$ accuracy). We allocate the remaining error budget evenly between
classifier validation and Monte Carlo estimation by setting $\epsilon_{clf} = 0.02$ and $\epsilon_{ver} = \epsilon - (\hat e + \epsilon_{clf}) = 0.02$,
with confidence parameters $\delta_{clf} = \delta_{ver} = \delta/2 = 5 \cdot 10^{-4}$. By Theorem~\ref{thm:pac_guarantee}, the required validation set size for the classifier and the sample histories for the verification are $M, N \ge \frac{\ln(2/\delta)}{2\epsilon^2} = \frac{\ln(2000)}{2 \cdot 0.02^2}
\approx 10362$. Hence, \texttt{RNN-ProVe} certifies the violation volume with $99.9\%$ confidence and a maximum error margin of $5\%$ using $\approx2\cdot10362$ total queries. Crucially, parallel execution on a GPU makes this computation take negligible time as shown in our experiments.

\paragraph{Limitations.}
\texttt{RNN-ProVe} focuses on discrete state, action, and observation spaces, which is consistent with most existing RNN verification methods discussed in Sec. \ref{sec:preliminaries}. While the underlying theory is not inherently restricted to discrete domains, extending \texttt{RNN-ProVe} to continuous settings requires identifying feasible hidden states when state and observation spaces are continuous. A natural extension, following standard practice in verification, is to discretize or quantize these spaces to obtain a finite abstraction over which feasibility can be learned and verified \citep{LiuSurvey,eProve,RFProve}. An alternative is to approximate feasibility via conservative reachable-set approximations of the hidden-state dynamics. Exploring such extensions is left for future work.
\texttt{RNN-ProVe} may also be less effective for very small problem instances or when the rate of undesired behavior is extremely low. In these cases, the number of samples required to achieve high statistical confidence may approach or exceed the number of feasible histories, making direct enumeration more appropriate. More generally, since the sample complexity scales as $O(\epsilon^{-2})$, estimating extremely rare violations (where $p^*\to 0$) requires very small error tolerances, leading to large sample requirements. Nonetheless, this behavior reflects the inherent complexity of the underlying counting problem rather than a limitation of the proposed method.

%% file: sections/experiments.tex
\section{Experiments}
\label{sec:experiments}

Our experiments address the following points:
(i) \textit{Can \texttt{RNN-ProVe} verify recurrent RL policies in both partially observable single- and cooperative multi-agent settings?}
(ii) \textit{Does \texttt{RNN-ProVe} scale with increasing problem complexity, in regimes where existing linearization-based RNN verification methods become computationally intractable?}
(iii) \textit{How does \texttt{RNN-ProVe} compare to naive sampling baselines, and how does incorporating \texttt{RNN-ProVe} dataset affect the linearization-based baseline?}

\paragraph{Implementation Details.} To answer these questions, we evaluate \texttt{RNN-ProVe} on two classes of discrete environments. The first consists of single-agent navigation tasks of increasing size, with policies trained using the widely adopted DRQN algorithm \citep{drqn}. The second consists of cooperative multi-agent policies trained on the BoxPushing (BP) domain \citep{yuchen_ma}, with increasing environment sizes, using the CTDE DRQN-based approach of \cite{yuchen_ma}. Both environments are illustrated in Fig. \ref{fig:tasks} and described in detail in Appendix \ref{appendix:task_description}. 


Since existing RNN verification tools are not designed to identify the feasible history space $\mathcal{H}^*$ and would therefore require an external oracle for this purpose, we adapt the linearization-based RNN verification approach of \cite{ko2019popqorn} as a baseline for comparison with \texttt{RNN-ProVe}. In particular, we use the same feasibility oracle as in \texttt{RNN-ProVe} and perform an exact quantitative volume estimation. Employing the same feasibility oracle ensures a fair and controlled comparison. Our experiments thus aim to highlight the scalability limitations of exact solvers and the practical advantages of \texttt{RNN-ProVe} in challenging settings. Notably, both \texttt{RNN-ProVe} and the baseline are agnostic to the specific learning algorithm used to train the policies, as is standard in verification. In all experiments, the recurrent component of the policy is implemented using a GRU layer, selected for its strong empirical performance.
Experiments are conducted on Xeon E5-2650 CPU nodes with $64$\,GB of RAM. We verify $10$ independently trained policies per task, collected at convergence. The following curves show the average return over the independent runs, smoothed over $100$ episodes, with shaded regions indicating $95\%$ confidence intervals. We also trained three independent feasibility classifiers with different random initializations; as the resulting verification estimates exhibited negligible variation, we report only the averaged results in Tab.~\ref{tab:tab_eval}. Complete RL and classifier hyperparameters are selected via grid search and reported in Appendix \ref{appendix:hyperparams}. The environmental impact of the experiments is discussed in Appendix \ref{appendix:env_impact}.

\paragraph{Desired Behaviors.}
For single-agent navigation, we verify collision-avoidance behaviors (Def. \ref{def:sharp_rnn}), where the agent should not select an action that leads to a collision with an obstacle, as in \cite{CROP, eps_retrain, ral}. In the cooperative BP, we verify a joint behavior (Def. \ref{def:sharp_rnn_marl}), requiring that both agents pick the push action when positioned in front of the box. These are designed to assess both individual and cooperative violations. At a high-level, the resulting violation percentages can be interpreted as the likelihood that an agent exhibits undesired behavior when reaching a given state from arbitrary feasible histories, providing an indicator of policy generalization and robustness to history-dependent uncertainty, such as that induced by noisy or non-deterministic interactions.

\input{tables/verification_results}

\begin{figure}[t]
    \centering
    \includegraphics[width=0.75\linewidth]{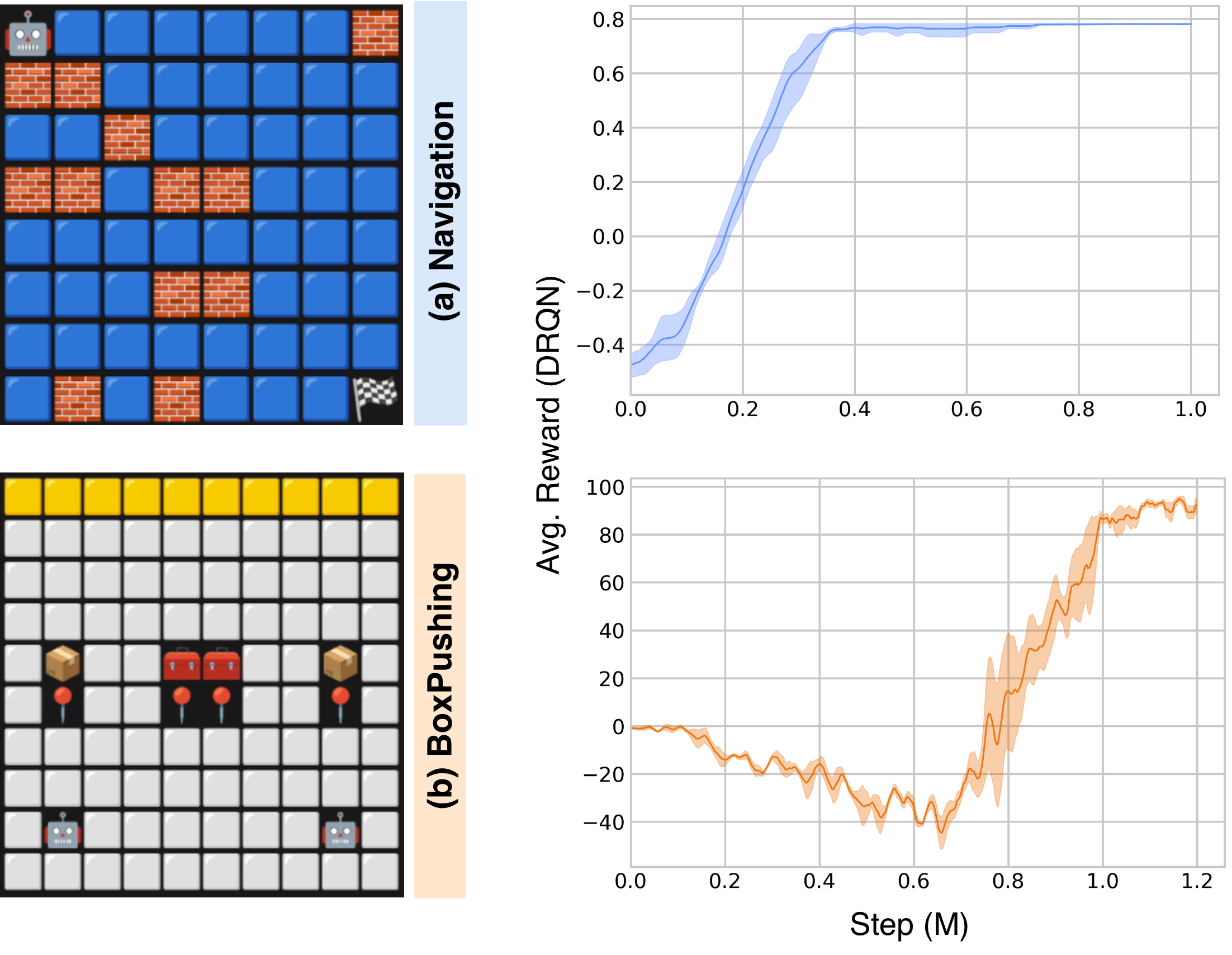}
    \caption{(Top) Single-agent navigation with red obstacles, where the agent must navigate to the goal location (left) and the corresponding training performance of DRQN (right). (Bottom) Cooperative multi-agent BP (left), where agents must jointly push a box to a yellow goal location, and the corresponding joint training performance of the agents using a CTDE DRQN-based algorithm (right).}
    \vspace{-7pt}
    \label{fig:tasks}
\end{figure}

\subsection{Empirical Evaluation}

\paragraph{Verification of RNN-based policies.}

\texttt{RNN-ProVe} verifies recurrent policies in all single-agent navigation and cooperative multi-agent BP. In navigation tasks, \texttt{RNN-ProVe} identifies non-trivial violation rates in smaller environments, with a mean violation of approximately $1\%$ in the $4\times4$ grid and approximately $13\%$ in the $8\times8$ grid (Table \ref{tab:tab_eval}). In contrast, for the $16\times16$ environment, \texttt{RNN-ProVe} reports low violation rates below $1\%$ under the specified $(\epsilon,\delta)$ guarantees. 

Fig. \ref{fig:heatmap} also visualizes the spatial distribution of violations for the $8\times8$ environment. Despite strong training performance, the learned policies do not generalize uniformly across feasible hidden states. Specifically, the reported violation values correspond to the fraction of feasible histories from which the agent selects an action that eventually leads to a collision, revealing sensitivity to particular configurations near obstacles. These failures are not captured by standard empirical evaluations based on average return \citep{marzari2026varepsilon}.
In the cooperative BP task, \texttt{RNN-ProVe} certifies mean violation values of approximately $1.18\%$ in the $10\times10$ environment and approximately $1.7\%$ in the $20\times20$ environment, demonstrating that the framework naturally extends to cooperative settings with interacting recurrent agents while providing quantitative, probabilistic safety assessments.

\begin{figure}[b!]
    \centering
    \includegraphics[width=.45\linewidth]{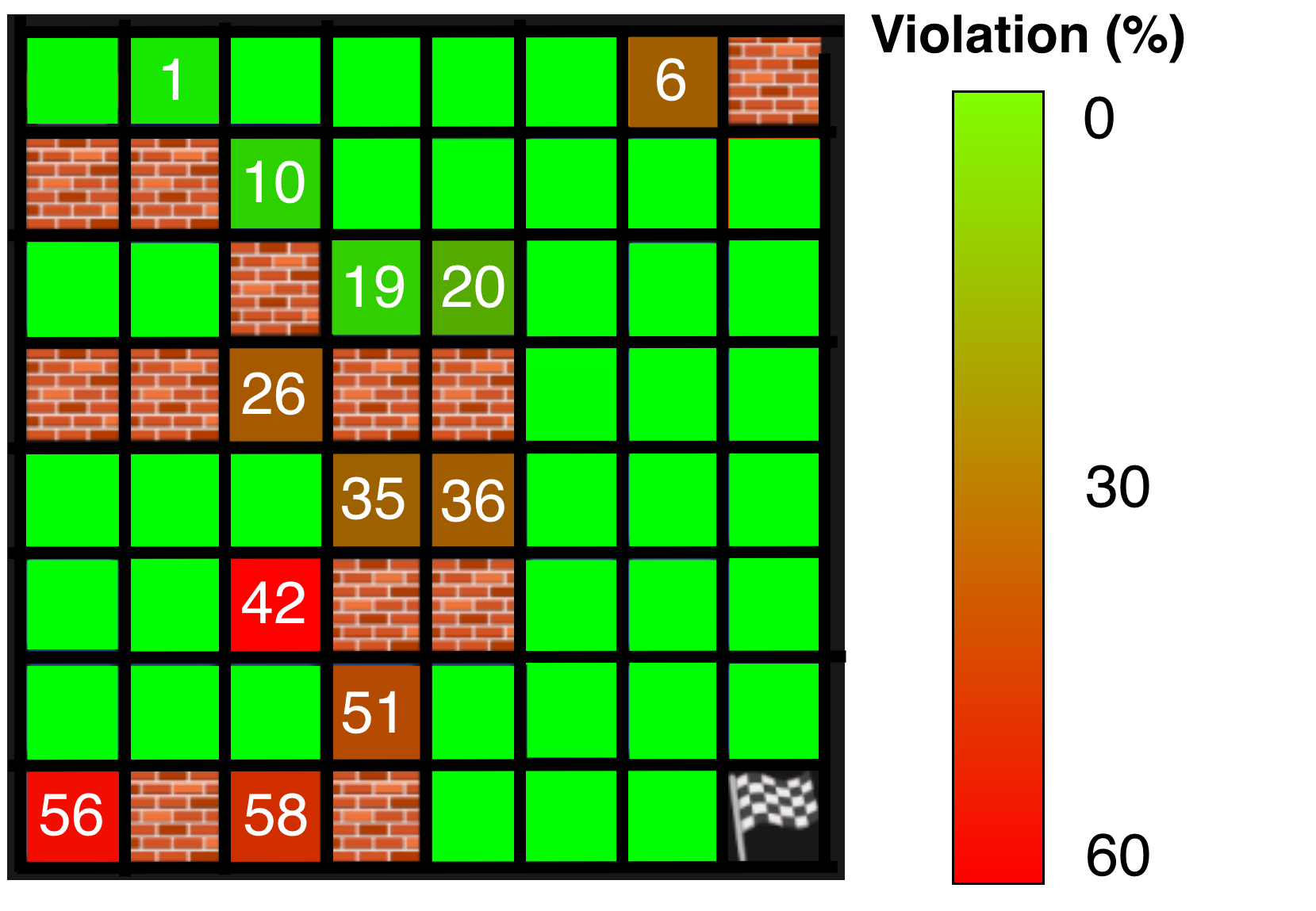}
    \caption{Violation heatmap for the $8\times8$ navigation task. Each grid cell (state) is colored according to the \% of violations computed via \texttt{RNN-ProVe}. Red cells indicate higher violations, revealing regions of the environment where the learned policy is more sensitive.} 
    \label{fig:heatmap}
\end{figure}

\paragraph{Scalability Consideration.}
Results in Tab. \ref{tab:tab_eval} highlight the scalability limitations of exact linearization-based RNN verification when applied to recurrent RL policies. In the navigation task, the baseline requires several minutes of computation to verify the small $4\times4$ environment, and incurs a substantial increase in runtime, exceeding $18$ minutes, in the $8\times8$ setting. For the $16\times16$ environment, it fails to complete within the imposed $45$-minute timeout. This behavior reflects the rapid growth of exact over-approximation-based RNN verification methods as both the environment size and the GRU hidden state dimension increase. In contrast, \texttt{RNN-ProVe} scales efficiently across all tasks, returning quantitative verification results (comparable with the baseline) in milliseconds. This efficiency stems from reasoning probabilistically over feasible histories. Moreover, existing RNN verification methods are not designed to handle cooperative multi-agent formulations and therefore cannot be directly applied to the BP task. In these settings, \texttt{RNN-ProVe} enables probabilistically quantitative verification of recurrent policies under decentralized execution, demonstrating its applicability beyond single-agent scenarios.

\paragraph{Ablation Study.}  
We report two additional experimental variants for the navigation tasks in Table~\ref{tab:tab_eval}. For the $4 \times 4$ environment, we compare baseline results obtained with an exact feasibility oracle to those obtained with the approximate feasibility oracle learned by \texttt{RNN-ProVe}. In this environment, where an exact oracle is achievable, we notice that the verification results are nearly identical, indicating that the histories collected during training already cover a large portion of the feasible history space. When used within the \texttt{RNN-ProVe} pipeline, this data enables learning a feasibility classifier that closely approximates the true feasible-history distribution. 
We also compare our approach against a naive \emph{Monte-Carlo} variant that removes the feasibility classifier and evaluates the policy on uniformly sampled hidden states. This leads to substantially higher violation rates, as most sampled hidden states do not correspond to realizable interaction histories and therefore induce spurious failures when fed into the policy. Overall, these experiments rule out two natural alternatives: naive linearization and bound-propagation–based hidden-state bound estimation, as well as naive Monte Carlo estimation without feasibility filtering. They further show that the proposed components are not only necessary to verify histories arising under policy execution, but also sufficient to produce accurate probabilistic estimates of behavioral violations. By operating over feasible combinations of hidden-state features, \texttt{RNN-ProVe} computes violation values that reliably reflect the true behavior of the policy.

%% file: tables/verification_results.tex
\begin{table*}[t]
    \scriptsize
    \centering
    \begin{tabular}{lccccccccccc}
    
    \textbf{Method} & \textbf{Task} & \textbf{Env size} & \textbf{GRU size} & \textbf{Feas. Oracle} & $\boldsymbol{1-\hat{e}}$ & \textbf{\# samples} & $\lvert \mathbf{h}^* \rvert$ & \textbf{Avg. Violation} & $\boldsymbol{1-\delta}$ & $\boldsymbol{\epsilon}$ & \textbf{Time (s)}\\
    \Xhline{1pt}
    \addlinespace[2pt]
    \rowcolor{gray!20}
    Baseline & Nav & 4$\times$4 & 4 & Exact &- & - & - & 1.44\% & - & - & 223.13\\
    \rowcolor{gray!20}
    Baseline  & Nav & 4$\times$4 & 4 &  Approx. & 99.97 & - & - & 1.43\% & - & - & 215.58\\
    \rowcolor{gray!20}
    Monte-Carlo & Nav & 4$\times$4 & 4 & - &- & 1M & 1M & 23.87\% & - & - & 0.0244\\
    \texttt{RNN-ProVe} & Nav & 4$\times$4 & 4 & Approx. & 99.97 & 100k & 12637 & 1.62\% & 99\% & 3.11\% & 0.0053\\
    \texttt{RNN-ProVe} & Nav & 4$\times$4 & 4 & Approx. & 99.97 & 1M & 127359 & 1.42\% & 99\% & 1.0\% & 0.0271\\
    \midrule
    \rowcolor{gray!20}
    \rowcolor{gray!20}
    Baseline & Nav & 8$\times$8 & 8 & Approx. & 99.99 & - & - & 13.01\% & - & - & 1050.53\\
    \rowcolor{gray!20}
    Monte-Carlo & Nav & 8$\times$8 & 8 & - &- & 1M & 1M & 21.24\% & - & - & 0.0244\\
    \texttt{RNN-ProVe} & Nav & 8$\times$8 & 8 & Approx. & 99.99 & 100k & 10890 & 13.04\% & 99\% & 3.33\% & 0.0024\\
    \texttt{RNN-ProVe} & Nav & 8$\times$8 & 8 & Approx. & 99.99 & 1M & 110293 & 13.37\% & 99\% & 1.05\% & 0.0207\\
    \midrule
    \rowcolor{gray!20}
    \rowcolor{gray!20}
    Baseline & Nav & 16$\times$16 & 12 & Approx. & 98.36 & - & - & - & - & - & Timeout\\
    \rowcolor{gray!20}
    Monte-Carlo & Nav & 16$\times$16 & 12& - & - & 1M & 1M & 4.94\% & - & - & 0.0244\\
    \texttt{RNN-ProVe} & Nav & 16$\times$16 & 12 & Approx. & 98.36 & 100k & 12775 & 0.64\% & 99\% & 4.7\% & 0.0028\\
    \texttt{RNN-ProVe} & Nav & 16$\times$16 & 12 & Approx. & 98.36 & 1M & 127015 & 0.79\% & 99\% & 2.61\% & 0.0223\\
    \Xhline{1pt}
    \addlinespace[2pt]
    \texttt{RNN-ProVe} & BP & 10$\times$10 & 16 & Approx. & 99.98 $\vert$ 99.95 & 100k & 20320 & 1.15\% & 99\% & 2.45\% & 0.0072\\
    \texttt{RNN-ProVe} & BP & 10$\times$10 & 16 & Approx. & 99.98 $\vert$ 99.95 & 1M & 41061 & 1.18\% & 99\% & 1.76\% & 0.0179\\
    \midrule
    \texttt{RNN-ProVe} & BP & 20$\times$20 & 32& Approx. & 99.23 $\vert$ 99.99 & 100k & 32413 & 1.51\% & 99\% & 2.69\% & 0.0140\\
    \texttt{RNN-ProVe} & BP & 20$\times$20 & 32& Approx. & 99.23 $\vert$ 99.99 & 1M & 65027 & 1.73\% & 99\% & 2.12\% & 0.0476\\
    \Xhline{1pt}
    \end{tabular}
    \caption{Quantitative verification results for RNN-based navigation (Nav) and BP tasks. Gray rows correspond to the baselines; \texttt{RNN-ProVe} reports probabilistic certificates with bounded error $\epsilon$ and confidence at least $1-\delta$.}
    \label{tab:tab_eval}
    \vspace{-3mm}
\end{table*}

%% file: sections/discussion.tex
\section{Discussion}

We presented \texttt{RNN-ProVe}, a probabilistic verification framework for recurrent policies in single- and cooperative multi-agent RL. \texttt{RNN-ProVe} addresses the challenge of reasoning over \emph{feasible histories}, which is not captured by existing RNN verification methods. By combining policy-driven sampling with statistical analysis, we produce quantitative, high-confidence certificates that estimate how frequently undesired behaviors occur over feasible hidden states. 
\texttt{RNN-ProVe} enables scalable analysis in regimes where exact methods become intractable. Our certificates depend on the coverage of feasible histories induced during training and on the accuracy of the learned feasibility classifier, dependencies that are unavoidable in data-driven settings but  characterized through bounded-error, high-confidence guarantees.

A key insight emerging from our analysis is that verification precision is governed not by tighter bounds, but by accurately capturing the combinatorial structure of feasible hidden state configurations induced by policy execution.


%% file: sections/appendix.tex
\section*{Appendix}

\section{Numerical example}
\label{appendix:numerical_example}
In the following, we illustrate interval-based robustness \textsc{RNN-Verification} using the toy RNN in Fig. \ref{fig:rnn_suppl}, assuming scalar weights $W_x = W_y = W_h = 1$ and ReLU activation. 

\begin{figure*}[h]
    \centering
    \vspace{-5pt}
    \includegraphics[width=0.5\linewidth]{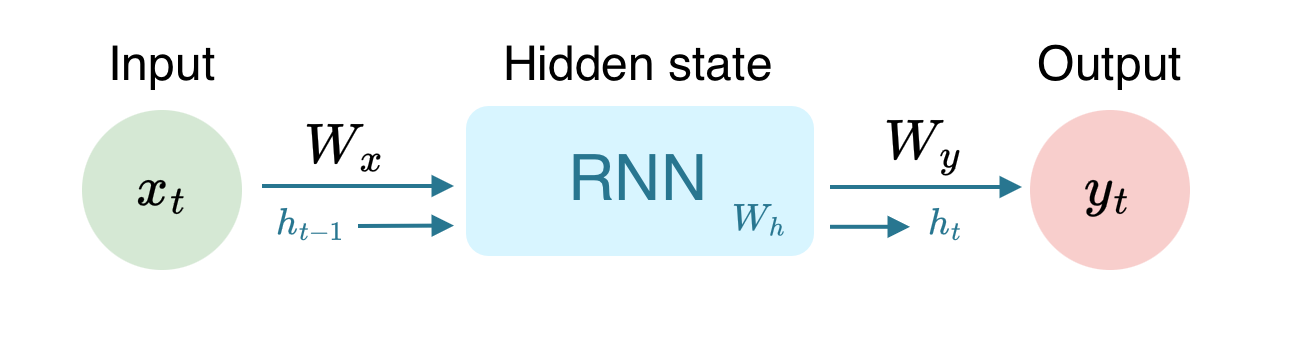}
    \vspace{-8pt}
    \caption{An explanatory RNN transition.}
    \label{fig:rnn_suppl}
\end{figure*}

\begin{example}
    Let the nominal input be $\mathbf{x} = [1, 2]^\top$ (i.e., an input sequence of length 2) and perturbation radius $\varepsilon = 1$, yielding intervals $x_1 \in [0, 2]$ and $x_2 \in [1, 3]$. We initialize $\mathcal{H}_0 = \{0\}$. The verification proceeds by propagating intervals through time. For t=1, we obtain $h_1 = \mathrm{ReLU}(h_0 + x_1) = \mathrm{ReLU}(0 + [0, 2]) = [0, 2]$. Similarly for $t=2$, we have $h_2 = \mathrm{ReLU}(h_1 + x_2) = \mathrm{ReLU}([1, 5]) = [1, 5]$. Finally, $\mathcal{Y}_T = y_T = 1 \cdot h_2 = [1, 5]$. Since $\min(\mathcal{Y}_T) = 1 > 0$, the RNN is verified as robust.
\end{example}
\section{On the Hardness of the \textsc{\#RNN-Verification} Problem.} \label{appendix:hardness}

We now analyze the computational complexity of the \textsc{\#RNN-Verification} problem. 

\setcounter{theorem}{0}
\begin{theorem}
    The \textsc{\#RNN-Verification} problem is \#P-hard.
\end{theorem}

\begin{proof}[Proof Sketch]
    The hardness result relies on the established relationship between Boolean satisfiability and Neural Network verification. In detail, \citet{Reluplex} demonstrated a polynomial-time reduction from \textsc{3-SAT} to the problem of finding a satisfying input (i.e., counterexample) for a deep neural network. Given a formula $\phi$ with $n$ variables, one can construct a neural network $N_\phi: [0, 1]^n \to \mathbb{R}$ such that $N_\phi(x) > 0$ if and only if $x$ corresponds to a satisfying assignment of $\phi$. Recent work on the \textsc{\#DNN-Verification} problem \citep{CountingProve} extends this reduction to the counting domain, showing a reduction from the \textsc{\#3-SAT} \citep{valiant1979complexity}. Specifically, computing the volume of the input domain $\mathcal{X}$ for which a network satisfies a property (i.e., measuring the set $\{x \in \mathcal{X} \mid N_\phi(x) > 0\})$ is equivalent to counting the satisfying assignments of $\phi$. Thus, the exact volumetric verification of DNNs is \#P-hard.

    Let us consider a specific instance of \textsc{\#RNN-Verification} with a time horizon $T=1$. The problem requires characterizing the set of histories (hidden states) $\Gamma(\mathcal{T}) := \{h \in \mathcal{H} \mid f(h, x) \le 0\}$. If we fix the observation $x$ and consider $h$ as the input variable, determining the volume $Vol(\Gamma(\mathcal{T}))$ is equivalent to determining the preimage volume of a DNN $f$ that maps to the unsafe region $(-\infty, 0]$. Hence, from the \#P-hardness of the \textsc{\#DNN-Verification} it follows that \textsc{\#RNN-Verification} is at least as hard as \#3-SAT. Consequently, \textsc{\#RNN-Verification} is \#P-hard.
\end{proof}

\section{Tasks Description}\label{appendix:task_description}

\begin{figure}[h!]
    \centering
    \includegraphics[width=\linewidth]{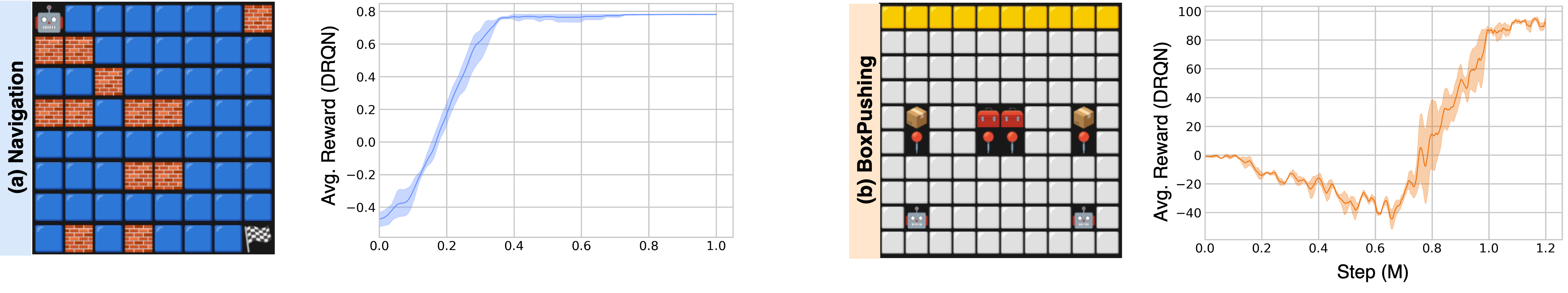}
    \caption{Environments and training performance. (Left) Single-agent navigation task with obstacles in red and the corresponding training performance of a DRQN policy, measured by average episode reward. (Right) Cooperative multi-agent BoxPushing task, where agents must jointly push a box to a goal location in yellow, and the corresponding joint training performance of the agents using a CTDE DRQN-based policy.}
    \label{fig:tasks}
\end{figure}

\paragraph{Single-agent navigation.}
We consider a partially observable single-agent navigation task defined on a discrete grid world. The agent starts in the top-left corner of the grid and must reach a fixed goal located in the bottom-right corner while avoiding collisions with static obstacles. Obstacles occupy approximately $20\%$ of the grid cells and are placed randomly, subject to the constraint that at least one collision-free path to the goal exists. At each time step, the agent can move in one of the four cardinal directions and receives a local observation that reveals the contents of the cells immediately adjacent to its current position, indicating whether neighboring cells contain free space (blue), obstacles (red), or the goal. The remainder of the grid is unobserved, requiring the agent to rely on its internal hidden state to integrate information over time. We evaluate increasingly complex variants of this task with grid sizes of $4\times4$, $8\times8$, and $16\times16$, which induce progressively larger hidden state spaces in the recurrent policy. An overview of the $8\times8$ environment and the training performance of the learned policy are shown in Fig. \ref{fig:tasks} (left).

\paragraph{Multi-agent BoxPushing.}
We next consider the fully cooperative multi-agent BoxPushing task \citep{yuchen_ma}. In this environment, two agents must coordinate to push a large box, initially located near the center of the grid, to a designated goal region marked in yellow. The box cannot be moved by a single agent and therefore requires agents to push simultaneously from appropriate positions, indicated by red pins in Fig. \ref{fig:tasks} (bottom). At each time step, agents receive highly partial observations that reveal only the contents of the cell immediately in front of them, making coordination strongly dependent on the learned hidden states and past interactions. Agents are equipped with a set of $n$-step navigation actions that move an agent to the front of each box, a push action that moves the box in front of an agent in a given direction until it reaches a wall, and single-step movements in the four cardinal directions. We evaluate increasingly challenging variants of this task by considering grid sizes of $10\times10$ and $20\times20$, which substantially increase both the coordination complexity among agents and the dimensionality of the recurrent hidden states that must be verified. An overview of the $10\times10$ BoxPushing environment and the joint training performance of the cooperating agents are shown in Fig. \ref{fig:tasks} (right).

\section{Hyperparameters} \label{appendix:hyperparams}
Table \ref{tab:hyperparams} lists the hyperparameters considered in our initial grid search and the final (best-performing) parameters used for training the policies and the classifier evaluated in Sec. \ref{sec:experiments}. 

\input{tables/hyperparams}
 
\section{Environmental Impact} \label{appendix:env_impact}
Although each individual training run is relatively computationally inexpensive, the overall experimental evaluation required running a large number of experiments on shared computing nodes, leading to cumulative environmental impacts. All experiments were conducted on a private infrastructure with an estimated carbon efficiency of approximately $0.275\,\mathrm{kgCO_2eq/kWh}$. We estimated the associated carbon emissions using the Machine Learning Impact calculator\footnote{\url{https://mlco2.github.io/impact\#compute}} and offset the resulting footprint through the Treedom platform.\footnote{\url{https://www.treedom.net}} We do not explicitly estimate or offset other categories of environmental impact, such as water usage or embodied emissions associated with hardware production, although we acknowledge their importance.

%% file: tables/hyperparams.tex
\begin{table}[h]
\centering
\caption{Hyper-parameters candidate for initial grid search tuning. Best performing parameters used for training are highlighted in bold.}
\begin{tabular}{ll}
\toprule
Learning rate                  & 3e-3, \textbf{3e-4}, 3e-5 \\
$\gamma$                       & \textbf{0.9}, 0.95, 0.99        \\
Buffer size (episodes)                    & \textbf{1000}, 2500, 5000            \\

Batch size                     & \textbf{32}, 64, 128            \\
Sampling trajectory size       & 10, 25, 50, \textbf{\textit{T}}             \\
Polyak averaging $\omega$      & \textbf{0.995}, 0.9998          \\
N° hidden layers               & \textbf{2}, 3                   \\
Hidden layers size             & \textbf{32}, 64, 128                \\ \midrule
$\mathcal{C}$ Learning rate                  & 3e-3, 3e-4, \textbf{3e-5} \\
$\mathcal{C}$ training epochs                  & 500, \textbf{1000}, 2000 \\
$\mathcal{C}$ N° hidden layers               & \textbf{2}, 3                   \\
$\mathcal{C}$ hidden layers size            & 32, \textbf{64}, 128                \\
$\mathcal{C}$ training-test split            & \textbf{80-20}               \\

\bottomrule
\end{tabular}
\label{tab:hyperparams}
\end{table}

%% file: ijcai26.bib
@inproceedings{RFProve,
  title={On the probabilistic learnability of compact neural network preimage bounds},
  author={Marzari, Luca and Bicego, Manuele and Cicalese, Ferdinando and Farinelli, Alessandro},
  booktitle={Proceedings of the AAAI Conference on Artificial Intelligence},
  volume={40},
  number={42},
  pages={35707--35714},
  year={2026}
}

@inproceedings{ModelVerification,
  author       = {Tianhao Wei and
                  Hanjiang Hu and
                  Luca Marzari and
                  Kai S. Yun and
                  Peizhi Niu and
                  Xusheng Luo and
                  Changliu Liu},
  title        = {ModelVerification.jl: {A} Comprehensive Toolbox for Formally Verifying
                  Deep Neural Networks},
  booktitle    = {Computer Aided Verification - 37th International Conference, {CAV}
                  2025},
  series       = {Lecture Notes in Computer Science},
  volume       = {15932},
  pages        = {395--408},
  publisher    = {Springer},
  year         = {2025},
  url          = {https://doi.org/10.1007/978-3-031-98679-6_18},
  doi          = {10.1007/978-3-031-98679-6_18}
}

@inproceedings{CountingProve,
  author       = {Luca Marzari and
                  Davide Corsi and
                  Ferdinando Cicalese and
                  Alessandro Farinelli},
  title        = {The {\#}DNN-Verification Problem: Counting Unsafe Inputs for Deep
                  Neural Networks},
  booktitle    = {International Joint Conference on
                  Artificial Intelligence, {IJCAI}},
  pages        = {217--224},
  year         = {2023},
  doi          = {10.24963/IJCAI.2023/25}
}

@inproceedings{tran2023rnnreach,
  title={Reachability Analysis of Recurrent Neural Networks Using Star Sets},
  author={Tran, Hoang-Dung and Johnson, Taylor T and others},
  booktitle={ACM International Conference on Hybrid Systems: Computation and Control},
  year={2023}
}

@inproceedings{carr2021verifiable,
  title={Verifiable RNN-based policies for POMDPs under temporal logic constraints},
  author={Carr, Steven and Jansen, Nils and Topcu, Ufuk},
  booktitle={International Conference on International Joint Conferences on Artificial Intelligence},
  pages={4121--4127},
  year={2021}
}

@article{everett2021reachability,
  title={Reachability analysis of neural feedback loops},
  author={Everett, Michael and Habibi, Golnaz and Sun, Chuangchuang and How, Jonathan P},
  journal={IEEE Access},
  volume={9},
  pages={163938--163953},
  year={2021},
  publisher={IEEE}
}

@article{hoeffding1963probability,
  title={Probability inequalities for sums of bounded random variables},
  author={Hoeffding, Wassily},
  journal={Journal of the American statistical association},
  volume={58},
  number={301},
  pages={13--30},
  year={1963},
  publisher={Taylor \& Francis}
}

@article{hornik1989multilayer,
  title={Multilayer feedforward networks are universal approximators},
  author={Hornik, Kurt and Stinchcombe, Maxwell and White, Halbert},
  journal={Neural networks},
  volume={2},
  number={5},
  pages={359--366},
  year={1989},
  publisher={Elsevier}
}

@inproceedings{akintunde2019verification,
  title={Verification of RNN-based neural agent-environment systems},
  author={Akintunde, Michael E and Kevorchian, Andreea and Lomuscio, Alessio and Pirovano, Edoardo},
  booktitle={AAAI Conference on Artificial Intelligence},
  pages={6006--6013},
  year={2019}
}

@inproceedings{mohammadinejad2021diffrnn,
  title={DiffRNN: differential verification of recurrent neural networks},
  author={Mohammadinejad, Sara and Paulsen, Brandon and Deshmukh, Jyotirmoy V and Wang, Chao},
  booktitle={International Conference on Formal Modeling and Analysis of Timed Systems},
  pages={117--134},
  year={2021},
  organization={Springer}
}

@inproceedings{du2021cert,
  title={Cert-RNN: Towards Certifying the Robustness of Recurrent Neural Networks.},
  author={Du, Tianyu and Ji, Shouling and Shen, Lujia and Zhang, Yao and Li, Jinfeng and Shi, Jie and Fang, Chengfang and Yin, Jianwei and Beyah, Raheem and Wang, Ting},
  booktitle={ACM SIGSAC Conference on Computer and Communications Security},
  pages={15--19},
  year={2021}
}

@inproceedings{ko2019popqorn,
  title={POPQORN: Quantifying robustness of recurrent neural networks},
  author={Ko, Ching-Yun and Lyu, Zhaoyang and Weng, Lily and Daniel, Luca and Wong, Ngai and Lin, Dahua},
  booktitle={International Conference on Machine Learning},
  pages={3468--3477},
  year={2019}
}

@article{RL,
  title={Reinforcement learning: A survey},
  author={Kaelbling, Leslie Pack and Littman, Michael L and Moore, Andrew W},
  journal={Journal of Artificial Intelligence Research},
  volume={4},
  pages={237--285},
  year={1996}
}

@article{valiant1979complexity,
  title={The complexity of computing the permanent},
  author={Valiant, Leslie G},
  journal={Theoretical Computer Science},
  volume={8},
  number={2},
  pages={189--201},
  year={1979},
  publisher={Elsevier}
}

@inproceedings{jacoby2020verifying,
  title={Verifying recurrent neural networks using invariant inference},
  author={Jacoby, Yuval and Barrett, Clark and Katz, Guy},
  booktitle={International Symposium on Automated Technology for Verification and Analysis},
  pages={57--74},
  year={2020},
  organization={Springer}
}

@inproceedings{Reluplex,
  title={Reluplex: An efficient SMT solver for verifying deep neural networks},
  author={Katz, Guy and Barrett, Clark and Dill, David L and Julian, Kyle and Kochenderfer, Mykel J},
  booktitle={International Conference on Computer Aided Verification},
  pages={97--117},
  year={2017}
}

@inproceedings{acrown,
    title={{Fast and Complete}: Enabling Complete Neural Network Verification with Rapid and Massively Parallel Incomplete Verifiers},
    author={Kaidi Xu and Huan Zhang and Shiqi Wang and Yihan Wang and Suman Jana and Xue Lin and Cho-Jui Hsieh},
    booktitle={International Conference on Learning Representations},
    year={2021},
    url={https://openreview.net/forum?id=nVZtXBI6LNn}
}

@inproceedings{ECAI25,
  author       = {Luca Marzari and
                  Isabella Mastroeni and
                  Alessandro Farinelli},
  title        = {Advancing Neural Network Verification Through Hierarchical Safety
                  Abstract Interpretation},
  booktitle    = {{ECAI} 2025 - 28th European Conference on Artificial Intelligence},
  pages        = {1736--1743},
  year         = {2025},
  doi          = {10.3233/FAIA251002},
}

@article{bcrown,
  title={Beta-crown: Efficient bound propagation with per-neuron split constraints for neural network robustness verification},
  author={Wang, Shiqi and Zhang, Huan and Xu, Kaidi and Lin, Xue and Jana, Suman and Hsieh, Cho-Jui and Kolter, J Zico},
  journal={Advances in Neural Information Processing Systems},
  volume={34},
  pages={29909--29921},
  year={2021}
}

@book{RNN,
  title={Recurrent neural networks: design and applications},
  author={Medsker, Larry and Jain, Lakhmi C},
  year={1999},
  publisher={CRC press}
}

@article{LSTM,
  title={Long short-term memory},
  author={Hochreiter, Sepp and Schmidhuber, J{\"u}rgen},
  journal={Neural computation},
  volume={9},
  number={8},
  pages={1735--1780},
  year={1997},
  publisher={MIT press}
}

@article{babSplitRelu,
  title={Branch and bound for piecewise linear neural network verification},
  author={Bunel, Rudy and Lu, Jingyue and Turkaslan, Ilker and Torr, Philip HS and Kohli, Pushmeet and Kumar, M Pawan},
  journal={Journal of Machine Learning Research},
  volume={21},
  number={42},
  pages={1--39},
  year={2020}
}

@article{INVPROP,
  title={Provably bounding neural network preimages},
  author={Kotha, Suhas and Brix, Christopher and Kolter, J Zico and Dvijotham, Krishnamurthy and Zhang, Huan},
  journal={Advances in Neural Information Processing Systems},
  year={2024}
}

@article{GRU,
  title={On the properties of neural machine translation: Encoder-decoder approaches},
  author={Cho, Kyunghyun and Van Merri{\"e}nboer, Bart and Bahdanau, Dzmitry and Bengio, Yoshua},
  journal={arXiv preprint arXiv:1409.1259},
  year={2014}
}

@inproceedings{zhang2024provable,
  title={Provable preimage under-approximation for neural networks},
  author={Zhang, Xiyue and Wang, Benjie and Kwiatkowska, Marta},
  booktitle={International Conference on Tools and Algorithms for the Construction and Analysis of Systems},
  pages={3--23},
  year={2024},
  organization={Springer}
}

@article{LiuSurvey,
  title={Algorithms for verifying deep neural networks},
  author={Liu, Changliu and Arnon, Tomer and Lazarus, Christopher and Strong, Christopher and Barrett, Clark and Kochenderfer, Mykel J and others},
  journal={Foundations and Trends{\textregistered} in Optimization},
  volume={4},
  number={3-4},
  pages={244--404},
  year={2021},
  publisher={Now Publishers, Inc.}
}

@article{pomdp,
title = {Optimal Control of Markov Processes with Incomplete State Information I},
author = {{\AA}str{\"o}m, Karl Johan},
year = {1965},
doi = {10.1016/0022-247X(65)90154-X},
language = {English},
volume = {10},
pages = {174--205},
journal = {Journal of Mathematical Analysis and Applications},
issn = {0022-247X},
publisher = {Academic Press}
}

@book{decpomdp,
    author = {Frans A Oliehoek and Christopher Amato},
    title = {A concise introduction to decentralized POMDPs},
    publisher = {Springer},
    year = {2016}
}

@inproceedings{papoudakis2021benchmarking,
title={Benchmarking Multi-Agent Deep Reinforcement Learning Algorithms in Cooperative Tasks},
author={Georgios Papoudakis and Filippos Christianos and Lukas Sch{\"a}fer and Stefano V Albrecht},
booktitle={Conference on Neural Information Processing Systems Datasets and Benchmarks Track},
year={2021},
}

@article{drqn,
    title={Deep Recurrent Q-Learning for Partially Observable MDPs}, 
    author={Matthew Hausknecht and Peter Stone},
    year={2017},
    journal={arXiv preprint arXiv:1507.06527}
}

@inproceedings{yuchen_ma,
    title = {Macro-Action-Based Deep Multi-Agent Reinforcement Learning},
    author = {Xiao, Yuchen and Hoffman, Joshua and Amato, Christopher},
    booktitle = {Conference on Robot Learning},
    year = {2020}
}

@article{amato_survey,
    title={An Introduction to Centralized Training for Decentralized Execution in Cooperative Multi-Agent Reinforcement Learning}, 
    author={Christopher Amato},
    year={2024},
    journal={arXiv preprint arXiv:2409.03052}
}

@inproceedings{eps_retrain,
  author       = {Luca Marzari and
                  Priya L. Donti and
                  Changliu Liu and
                  Enrico Marchesini},
  title        = {Improving Policy Optimization via {\(\epsilon\)}-Retrain},
  booktitle    = {Proceedings of the 24th International Conference on Autonomous Agents
                  and Multiagent Systems, {AAMAS} 2025},
  pages        = {1464--1472},
  publisher    = {International Foundation for Autonomous Agents and Multiagent Systems
                  / {ACM}},
  year         = {2025},
  url          = {https://dl.acm.org/doi/10.5555/3709347.3743780},
  doi          = {10.5555/3709347.3743780},
}

@inproceedings{eProve,
  author       = {Luca Marzari and
                  Davide Corsi and
                  Enrico Marchesini and
                  Alessandro Farinelli and
                  Ferdinando Cicalese},
  title        = {Enumerating Safe Regions in Deep Neural Networks with Provable Probabilistic
                  Guarantees},
  booktitle    = {Conference on Artificial Intelligence (AAAI)},
  pages        = {21387--21394},
  year         = {2024},
  doi          = {10.1609/AAAI.V38I19.30134},
}

@inproceedings{aamas23,
  author       = {Enrico Marchesini and
                  Luca Marzari and
                  Alessandro Farinelli and
                  Christopher Amato},
  title        = {Safe Deep Reinforcement Learning by Verifying Task-Level Properties},
  booktitle    = {Proceedings of the 2023 International Conference on Autonomous Agents
                  and Multiagent Systems, {AAMAS} 2023},
  pages        = {1466--1475},
  publisher    = {{ACM}},
  year         = {2023},
  url          = {https://dl.acm.org/doi/10.5555/3545946.3598799},
  doi          = {10.5555/3545946.3598799}
}

@inproceedings{CROP,
  author       = {Luca Marzari and
                  Enrico Marchesini and
                  Alessandro Farinelli},
  title        = {Online Safety Property Collection and Refinement for Safe Deep Reinforcement
                  Learning in Mapless Navigation},
  booktitle    = {{IEEE} International Conference on Robotics and Automation, ({ICRA})},
  pages        = {7133--7139},
  publisher    = {{IEEE}},
  year         = {2023},
  url          = {https://doi.org/10.1109/ICRA48891.2023.10161312},
  doi          = {10.1109/ICRA48891.2023.10161312},
}

@article{marchesini2025rl2grid,
  title={RL2Grid: Benchmarking Reinforcement Learning in Power Grid Operations},
  author={Marchesini, Enrico and Donnot, Benjamin and Crozier, Constance and Dytham, Ian and Merz, Christian and Schewe, Lars and Westerbeck, Nico and Wu, Cathy and Marot, Antoine and Donti, Priya L},
  journal={arXiv preprint arXiv:2503.23101},
  year={2025}
}

@ARTICLE{ral,
  author={Marzari, Luca and Trotti, Francesco and Marchesini, Enrico and Farinelli, Alessandro},
  journal={IEEE Robotics and Automation Letters}, 
  title={Designing Control Barrier Function via Probabilistic Enumeration for Safe Reinforcement Learning Navigation}, 
  year={2025},
  volume={10},
  number={10},
  pages={9630-9637},
  doi={10.1109/LRA.2025.3596431}}

@article{TIST, 
author = {Marzari, Luca and Cicalese, Ferdinando and Farinelli, Alessandro and Amato, Christopher and Marchesini, Enrico}, 
title = {Verifying Online Safety Properties for Safe Deep Reinforcement Learning}, 
year = {2025}, 
publisher = {Association for Computing Machinery}, 
issn = {2157-6904}, 
url = {https://doi.org/10.1145/3770068}, 
doi = {10.1145/3770068},
journal = {ACM Trans. Intell. Syst. Technol.}, 
month = sep}

@inproceedings{marl2grid_iclr2026,
    title={{MARL}2Grid-{TR}: A Multi-Agent {RL} Benchmark in Power Grid Operations},
    author={Enrico Marchesini and Eva Boguslawski and Alessandro Leite and Christopher Amato and Matthieu DUSSARTRE and Marc Schoenauer and Benjamin Donnot and Priya L. Donti},
    booktitle={The Fourteenth International Conference on Learning Representations},
    year={2026},
    url={https://openreview.net/forum?id=mpAMH1OyMO}
}

@inproceedings{aydeniz_2025,
author = {Aydeniz, Ayhan Alp and Marchesini, Enrico and Loftin, Robert and Amato, Christopher and Tumer, Kagan},
title = {Safe Entropic Agents under Team Constraints},
year = {2025},
booktitle = {International Conference on Autonomous Agents and Multiagent Systems (AAMAS)},
pages = {2411–2413},
}

@inproceedings{aydeniz_2024,
author = {Aydeniz, Ayhan Alp and Marchesini, Enrico and Amato, Christopher and Tumer, Kagan},
title = {Entropy Seeking Constrained Multiagent Reinforcement Learning},
year = {2024},
booktitle = {International Conference on Autonomous Agents and Multiagent Systems (AAMAS)},
pages = {2141–2143},
}

@article{marzari2026varepsilon,
  title={$\varepsilon$-retraining reinforcement learning algorithms},
  author={Marzari, Luca and Liu, Changliu and Donti, Priya L and Marchesini, Enrico},
  journal={Autonomous Agents and Multi-Agent Systems},
  volume={40},
  number={1},
  pages={24},
  year={2026},
  publisher={Springer}
}

@inproceedings{valuefact,
author = {Marchesini, Enrico and Baisero, Andrea and Bhati, Rupali and Amato, Christopher},
title = {On Stateful Value Factorization in Multi-Agent Reinforcement Learning},
year = {2025},
booktitle = {International Conference on Autonomous Agents and Multiagent Systems (AAMAS)},
pages = {1445–1453},
}
